\documentclass{bmvc2k}

\usepackage{amssymb,amsthm,amsmath,amsfonts,amstext}
\usepackage{times}
\usepackage{graphicx,subfig}
\usepackage{algorithm,algorithmic}
\usepackage{bm}
\usepackage{multirow}
\usepackage{nicefrac}
\usepackage{url}
\usepackage{caption}
\newtheorem{theorem}{Theorem}

\newtheorem{corollary}{Corollary}

\def\etal{\emph{et al}\bmvaOneDot}

\title{Network Decoupling: From Regular to Depthwise Separable Convolutions}

\addauthor{Jianbo Guo}{jianboguo@outlook.com}{1}
\addauthor{Yuxi Li}{lyxok1@sjtu.edu.cn}{2}
\addauthor{Weiyao Lin}{hellomikelin@gmail.com}{2}
\addauthor{Yurong Chen}{yurong.chen@intel.com}{3}
\addauthor{Jianguo Li}{jianguo.li@intel.com}{3}

\addinstitution{Institute for Interdisciplinary Information Sciences, Tsinghua University, China}
\addinstitution{Shanghai Jiaotong University, China}
\addinstitution{Intel Labs China}

\runninghead{Guo \etal}{Network Decoupling}

\begin{document}

\maketitle

\begin{abstract}
Depthwise separable convolution has shown great efficiency in network design, but requires time-consuming training procedure with full training-set available.
This paper first analyzes the mathematical relationship between regular convolutions and depthwise separable convolutions,
and proves that the former one could be approximated with the latter one in closed form. We show depthwise separable convolutions are principal components of regular convolutions. And then we propose network decoupling (ND), a training-free method to accelerate convolutional neural networks (CNNs) by transferring pre-trained CNN models into the MobileNet-like depthwise separable convolution structure, with a promising speedup yet negligible accuracy loss.
We further verify through experiments that the proposed method is orthogonal to other training-free methods like channel decomposition, spatial decomposition, etc. Combining the proposed method with them will bring even larger CNN speedup.
For instance, ND itself achieves about 2$\times$ speedup for the widely used VGG16, and combined with other methods, it reaches 3.7$\times$ speedup with graceful accuracy degradation.
We demonstrate that ND is widely applicable to classification networks like ResNet, and object detection network like SSD300.
\end{abstract}

%-------------------------------------------------------------------------

%-------------------------------------------------------------------------
\section{Introduction}
\label{sec:intro}
Convolutional neural networks (CNNs) demonstrate great success in various computer vision tasks, such as image classification~\cite{alexnet12}, object detection \cite{rcnn14}, image segmentation~\cite{fcn15}, etc. However, they suffer from high computation cost when deployed to resource-constrained devices.
Many efforts have been devoted to optimize/accelerate the inference speed of CNNs, which could be roughly divided into three categories.
\textit{First}, design-time network optimization considers designing efficient network structures from scratch in a handcraft way or automatic search way.
Typical handcraft based works include Xception~\cite{Chollet2016Xception}, MobileNet~\cite{howard2017mobilenets},
and networks with channel interleaving/shuffle~\cite{zhang2017igc,zhang2017shufflenet},
while typical works on automatic network architecture search are NASNet~\cite{zoph2016neural}, PNASNet~\cite{liu2017pnas}.

\textit{Second}, training-time network optimization takes pre-defined network structures as input, and refines the structures through regularized retraining or fine-tuning or even knowledge distilling \cite{hinton2015distilling}. Typical works involve
weight pruning \cite{han2015learning,han2015deep}, structure (filters/channels) pruning \cite{Wen2016Learning,Li2016Pruning,liu2017learning,luo2017thinet},
weight hashing/quantization \cite{hashnet}, low-bit networks \cite{Courbariaux2016BinaryNet,rastegari2016xnor}.

\textit{Third}, deploy-time network optimization takes pre-trained CNN models as input,
and replaces some redundant and less-efficient CNN structures with efficient ones in a training-free way. Low-rank decomposition \cite{Denton2014Exploiting}, spatial decomposition \cite{Jaderberg2014Speeding}, and channel decomposition \cite{Zhang2016Accelerating} fall into this category.
%Typical works include low-rank decomposition \cite{Denton2014Exploiting}, spatial \cite{Jaderberg2014Speeding} and channel \cite{Zhang2016Accelerating} decomposition.

Methods in the first two categories require time-consuming training procedure to produce desired outputs, with full training-set available.
On the contrary, methods in the third category may not require training-set at all, or in some cases require a small calibration-set (e.g., 5,000 images) to tune some parameters.
The optimization procedure can typically be done within dozens of minutes.
Hence, it is of great value when software/hardware vendors assist their customers to optimize CNN based solutions in case that either the time budget is so tight that training based solutions are not feasible, or the customer data are unavailable due to privacy or confidential issues.
Therefore, there is a strong demand for modern deep learning frameworks or hardware (GPU/ASIC/FPGA, etc) vendors to provide deploy-time model optimization tools.

Meanwhile, handcraft designed structures such as depthwise separable convolution \cite{Chollet2016Xception,howard2017mobilenets,zhang2017shufflenet}
have shown great efficiency over regular convolution, while still keeping high accuracy.
To the best of our knowledge, the mathematical relationship between regular convolutions and depthwise separable convolutions is not yet studied and unknown to the public.
Our motivation is to show their relationship, and present a solution to decouple the regular convolutions into depthwise separable convolutions in a training-free way for deploy-time network optimization/acceleration.
Our main contributions are summarized as below:
\begin{itemize}
\setlength{\topsep}{1pt}
\setlength{\itemsep}{1pt}
\setlength{\parskip}{1pt}
	\item We are the first to analyze and disclose the mathematical relationship between regular convolutions and depthwise separable convolutions.
    This theoretic result enables a lot of possibilities for future studies.
	\item We present a closed-form and data-free tensor decomposition to decouple regular convolutions into depthwise separable convolutions, and show that network decoupling (ND) enables noticeable speedup for different CNN models, such as VGG16 \cite{simonyan2014very}, ResNet \cite{he2016deep}, as well as object detection network SSD \cite{liu2016ssd}.
	\item We demonstrate that ND is complementary to other training-free methods like channel decomposition \cite{Zhang2016Accelerating}, spatial decomposition \cite{Jaderberg2014Speeding}, and channel pruning \cite{He2017Channel}.
	For instance, network decoupling itself achieves about $1.8\times$ speedup for VGG16. When combined with other training-free methods, it achieves $3.7\times$ speedup.
    \item We show extremely decoupled network is friendly to fine-tuning.
    The extremely decoupled network will bring more speedup but larger accuracy drop.
    For instance, $4\times$ speedup will bring >50\% accuracy drop, while $2\times$ speedup has <1\% drop.
    We show the larger accuracy drop can be recovered with just several epochs of fine-tuning.
\end{itemize}
\vspace{-3ex}

\section{Related Work}
Here we only discuss related works on deploy-time network optimization. Low-rank decomposition \cite{Denton2014Exploiting} exploits low-rank nature within CNN layers, and shows that fully-connected layers can be efficiently compressed and accelerated with low-rank decomposition, while convolutional layers can not.
Spatial decomposition \cite{Jaderberg2014Speeding} takes a single channel filter as input, and do per-channel spatial decomposition on regular convolution for each input/output channel, i.e., factorizing the $k_h\times k_w$  filter into $1\times k_w$ and $k_h\times 1$ filters.

Channel decomposition \cite{Zhang2016Accelerating} decomposes one conv-layer into two conv-layers, where the first one has the same filter-size but
with fewer channels, and the second one is a 1$\times$1 convolution.
Channel pruning \cite{He2017Channel} develops a training free method to prune useless filter channels by minimizing the response reconstruction error with a small-size calibration set.

Our motivation is different from all these methods, since we consider the possibility to decompose regular convolutions into depthwise separable convolutions.

\begin{figure*}[]
	\vspace{-0.1in}
	\centering
	\subfloat[Regular Convolution]{
		\label{fig:subfig:std_conv} %% label for first subfigure
		\includegraphics[height=0.17\linewidth, width=0.243\linewidth]{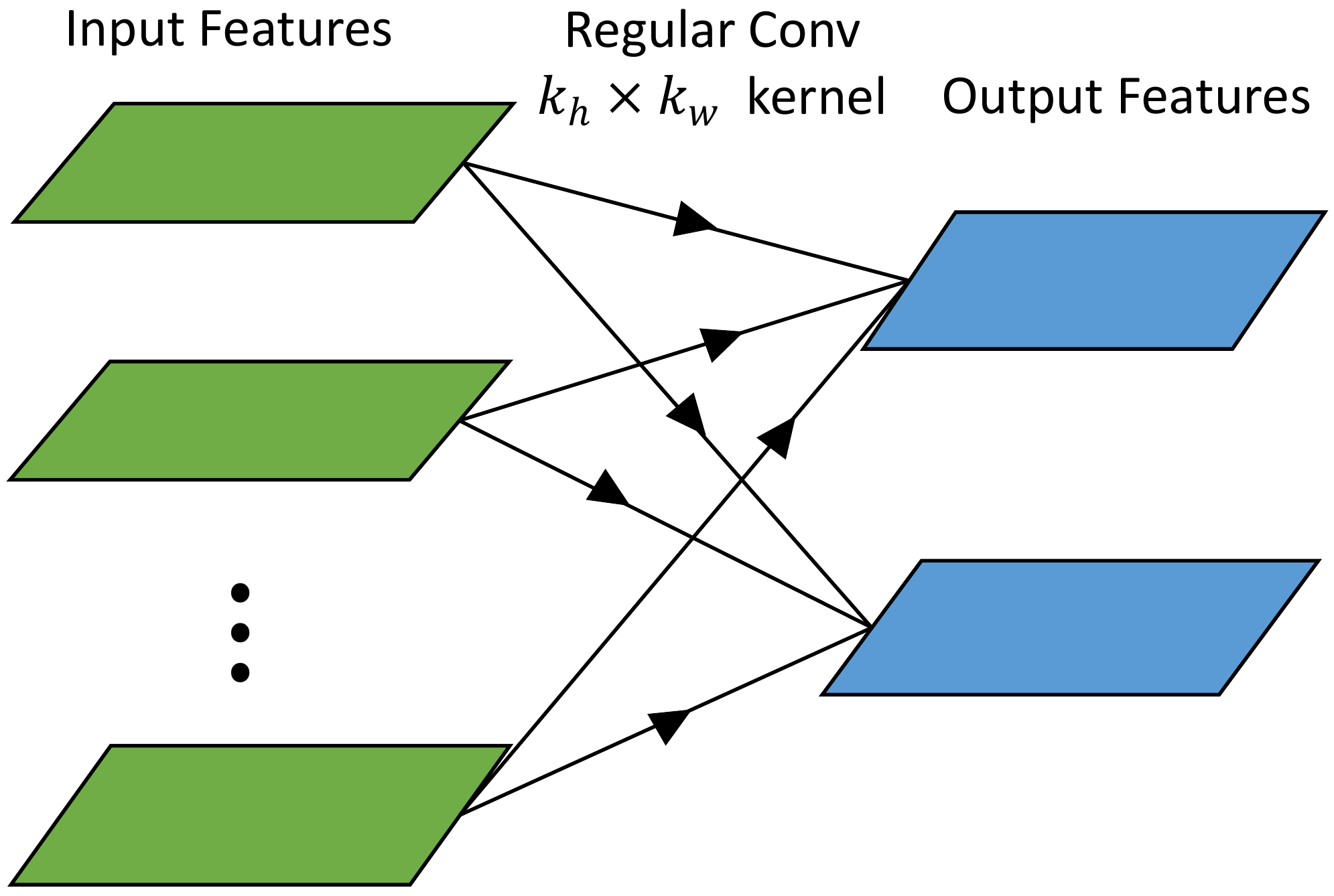}}
	\hspace{0.1ex}
	\subfloat[DW+PW Convolution]{
		\label{fig:subfig:dp_conv} %% label for first subfigure
		\includegraphics[height=0.17\linewidth, width=0.355\linewidth]{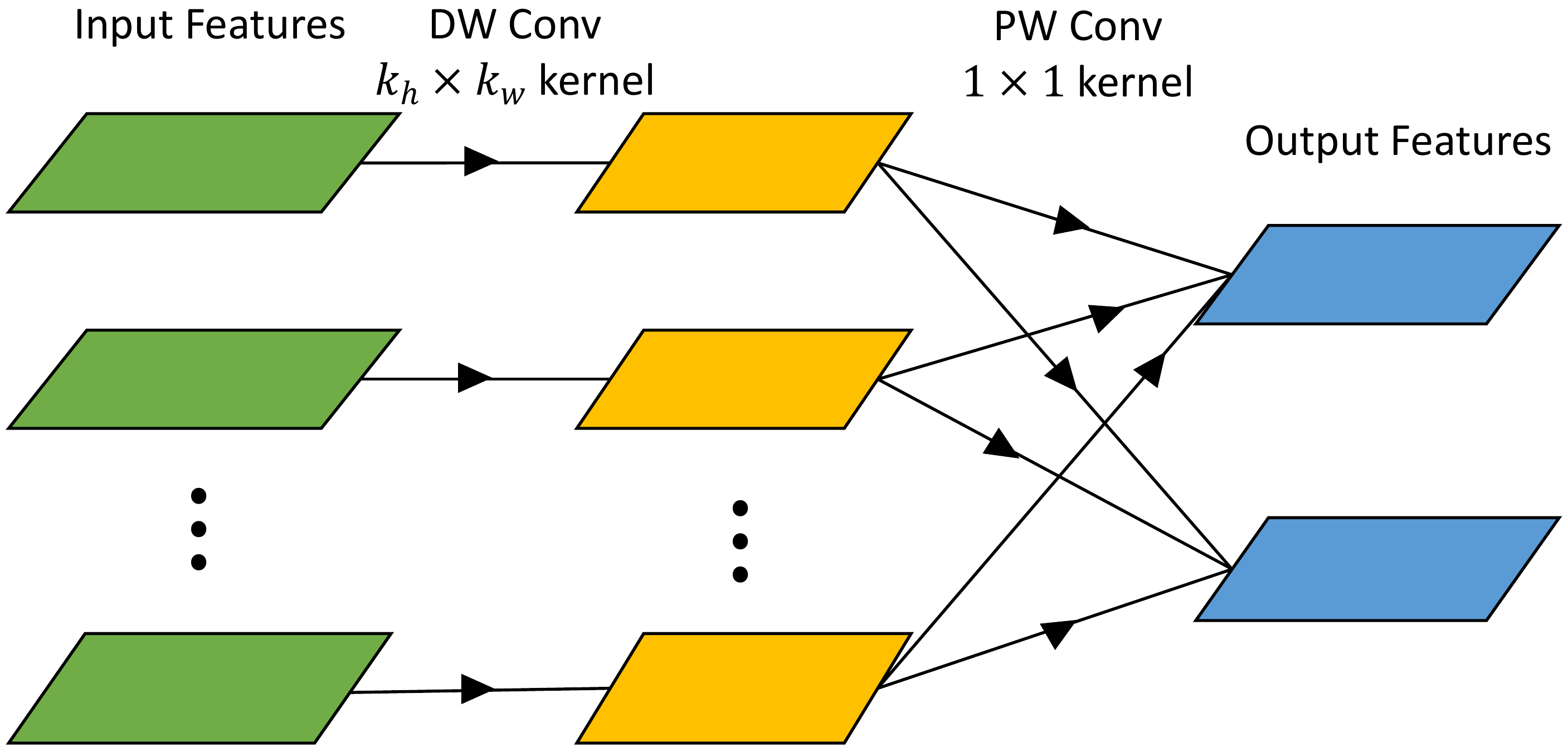}}
    \hspace{0.1ex}
	\subfloat[PW+DW Convolution]{
		\label{fig:subfig:pd_conv} %% label for first subfigure
		\includegraphics[height=0.17\linewidth, width=0.355\linewidth]{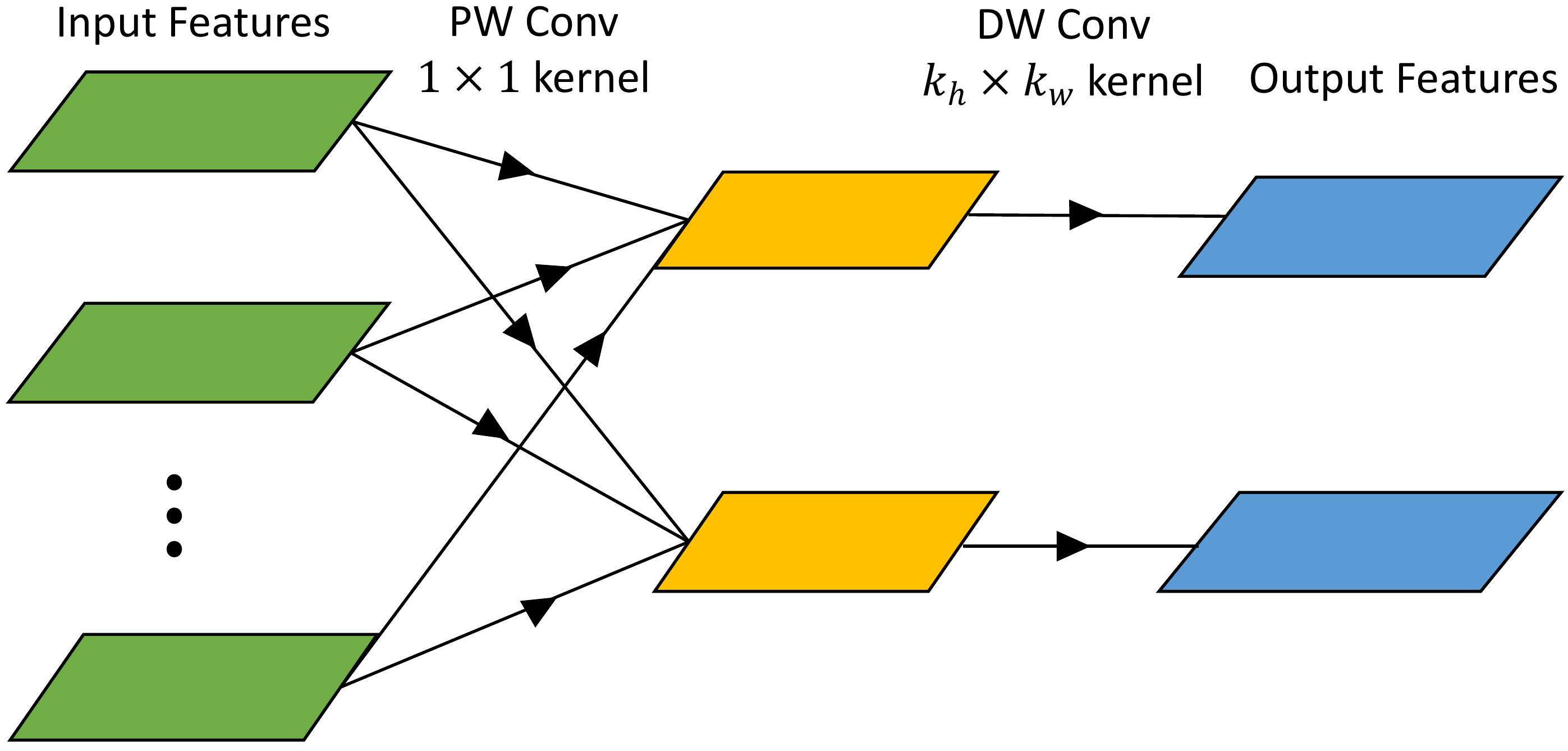}}
    \vspace{1.5ex}
	\caption{Regular convolutions vs. depthwise separable convolutions. (a) Regular convolution; (b) depthwise separable convolution in the DW+PW form (depthwise followed by pointwise); (c) depthwise separable convolution in the PW+DW form.}
	\label{fig:conv}
	\vspace{-0.21in}
\end{figure*}
\vspace{-0.5ex}
\section{Theory}
Our key insight is that different filter channels in regular convolutions are strongly coupled, and may involve plenty of redundancy.
Our analysis shows that this coupling induced redundancy is corresponding to some kind of low-rank assumption, with similar spirit of \cite{Denton2014Exploiting,Jaderberg2014Speeding,Zhang2016Accelerating}.
Here, we first analyze and disclose the mathematical relationship between regular convolutions and depthwise separable convolutions.

\vspace{-1ex}
\subsection{Regular vs. Depthwise Separable Convolutions}
A regular convolution kernel (Figure \ref{fig:subfig:std_conv}) is tasked to build both cross-channel correlation and spatial correlations.
Formally, we consider a convolution layer represented by a 4D tensor $\mathbf{W} \in \mathbb{R} ^{n_o\times n_i\times k_h \times k_w}$, where $n_o$ and $n_i$ are the number of output and input channels respectively, and $k_h$ and $k_w$ are the spatial height and width of the kernel respectively.
When the filter is applied to an input patch $\mathbf{x}$ with size $n_i \times k_h\times k_w$, we obtain a response vector $\mathbf{y} \in \mathbb{R}^{n_o}$ as
\begin{equation}
\small
\setlength{\abovedisplayskip}{1pt}
\setlength{\belowdisplayskip}{1pt}
    \mathbf{y} = \mathbf{W}*\mathbf{x},
\end{equation}
where $y_o = \sum_{i=1}^{n_i}{W}_{o, i}*x_i, o\in [n_o], i\in [n_i]$, and $*$ means convolution operation.
${W}_{o, i} = \mathbf{W}[o, i, :, :]$ is a tensor slice along the $i$-th input and $o$-th output channels,
$x_i = \mathbf{x}[i, :, :]$ is a tensor slice along the $i$-th channel of 3D tensor $\mathbf{x}$.
And the computational complexity for patch $\mathbf{x}$ is $O(n_o\times n_i\times k_h \times k_w)$.
It is easy to extend the complexity from patch level to feature map level.
Given the feature map size $H\times W$, the complexity is $O(H\times W \times n_o\times n_i\times k_h \times k_w)$.

Compared with the regular convolution, a depthwise separable convolution consists of a depthwise (DW) convolution followed by a pointwise (PW) convolution,
where DW focuses on spatial relationship modeling with 2D channel-wise convolutions, and PW focuses on cross-channel relationship modeling with $1\times 1$ convolution across channels.
This factorization form, denoted by DW+PW, is shown in Figure \ref{fig:subfig:dp_conv}.
To ensure the same shape output as the regular convolution $\mathbf{W}$, we set the DW convolution kernel tensor $\mathbf{D}\in \mathbb{R} ^{n_i\times 1\times k_h \times k_w}$, and the PW convolution tensor $\mathbf{P} \in \mathbb{R}^ {n_o\times n_i\times 1\times 1}$. When applying it to the input patch $\mathbf{x}$, we can obtain the corresponding response vector $\mathbf{y}'$ as
\begin{equation}
\small
\setlength{\abovedisplayskip}{1pt}
\setlength{\belowdisplayskip}{1pt}
\mathbf{y}' = (\mathbf{P}\circ \mathbf{D})*\mathbf{x},
\end{equation}
where $y'_o = \sum_{i=1}^{n_i} P_{o, i}\left(D_{i}*x_i\right)$, $\circ$ is the compound operation, $P_{o, i} = \mathbf{P}[o, i, :, :]$ and $D_i = \mathbf{D}[i, :, :, :]$.
And the computational complexity for the whole feature map is $O(H\times W\times(n_i\times k_h \times k_w + n_i\times n_o))$.

Alternatively, we could put PW convolution before DW and obtain another factorization form PW+DW as shown in Figure \ref{fig:subfig:pd_conv}.
In this case, $\mathbf{P} \in \mathbb{R} ^{n_o\times n_i\times 1\times 1}$, $\mathbf{D} \in \mathbb{R} ^{n_o \times 1\times k_h\times k_w}$.
When it is applied to the input patch $\mathbf{x}$, the response vector $y''$ is
\begin{equation}
\small
\setlength{\abovedisplayskip}{1pt}
\setlength{\belowdisplayskip}{1pt}
\mathbf{y}'' = (\mathbf{D}\circ \mathbf{P})*\mathbf{x},
\end{equation}
where $y''_o = D_o*\left(\sum_{i=1}^{n_i}P_{o, i}x_i\right)$, and $D_o = \mathbf{D}[o, :, :, :]$.
Here, the computational complexity is $O(H\times W\times (n_i\times n_o + n_o\times k_h \times k_w))$.
It is obvious that DW+PW and PW+DW are more efficient than regular convolutions according to the computational complexity.

\vspace{-1ex}
\subsection{Relationship}
We have shown regular convolutions model the spatial correlation and cross-channel correlation simultaneously with one tensor kernel,
while depthwise separable convolutions model these two correlations in a decoupling way.
Is there any relationship between these two convolution formulations?
Is it possible to approximate regular convolutions with depthwise separable convolutions precisely?
%And how does the redundancy in regular convolutions affect this approximation?
We give the following theorem to answer these questions.
\vspace{-4ex}
\begin{theorem}\label{th1}
	Regular convolutions can be losslessly expanded to a sum of several depthwise separable convolutions, without the increase of computational complexity. Formally, $\forall~ \mathbf{W}$ with spatial kernel size $k_h\times k_w$, $\exists$ $\{\mathbf{P}^k, \mathbf{D}^k\}_{k=1}^K$,
	\begin{equation}\label{eqndwpw}
\small
\setlength{\abovedisplayskip}{1pt}
\setlength{\belowdisplayskip}{1pt}
	\begin{split}
	s.t.~&~(a) K\le k_hk_w; \\
	&~(b) \mathbf{W} = \left\{ \begin{array}{ll}
	\sum\nolimits_{k=1}^K \mathbf{P}^k\circ \mathbf{D}^k  & \textrm{for~DW+PW}\\
	\sum\nolimits_{k=1}^K \mathbf{D}^k\circ \mathbf{P}^k  & \textrm{for~PW+DW}.
	\end{array}
	\right .
	\end{split}
	\end{equation}
\end{theorem}
\begin{proof}
 This problem is similar to the Kronecker product decomposition problem \cite{van2000ubiquitous,batselier2017constructive}, which factorizes a tensor into a linear combination of tensor/matrix Kronecker products. We adopt similar techniques to prove the above theorem. As DW+PW and PW+DW cases are similar, for simplicity, we only discuss the DW+PW case below.
	
	In the DW+PW case, $\mathbf{D}^k \in \mathbb{R}^{n_i\times 1\times k_h \times k_w}$ and $\mathbf{P}^k \in \mathbb{R}^{n_o\times n_i\times 1\times 1}$.
	Given an input patch $\mathbf{x}$, the response difference between regular convolution and DW+PW convolution is
		\begin{equation*}
    \small
    \setlength{\abovedisplayskip}{1pt}
    \setlength{\belowdisplayskip}{1pt}
		\begin{split}
		\|\mathbf{y}-\mathbf{y}'\|_2^2 &= \|(\mathbf{W} - \sum_{k=1}^K \mathbf{P}^k\circ \mathbf{D}^k)*x \|_2^2 = \sum_{o=1}^{n_o}(\sum_{i=1}^{n_i}W_{o, i}*x_i - \sum_{k=1}^K\sum_{i=1}^{n_i} P^k_{o, i}(D^k_{i}*x_i))^2\\
		&\le \sum_{o=1}^{n_o}\sum_{i=1}^{n_i}((W_{o, i} - \sum_{k=1}^K P^k_{o, i}D^k_{i})*x_i)^2 \le  \sum_{o=1}^{n_o}\sum_{i=1}^{n_i}\|W_{o, i} - \sum_{k=1}^K P^k_{o, i}D^k_{i}\|_F^2\cdot\|x_i\|_F^2,
		\end{split}
		\end{equation*}
	where $P^k_{o, i} = \mathbf{P}^k[o, i, :, :]$, $D^k_{i} = \mathbf{D}^k[i, :, :, :]$ are tensor slices, and $\Vert \cdot \Vert_F$ is the Frobenius norm.
	For the rightmost term,
	\begin{equation*}\label{pwdwsvd}
    \small
    \setlength{\abovedisplayskip}{1pt}
    \setlength{\belowdisplayskip}{1pt}
	\sum\limits_{o=1}^{n_o}\sum\limits_{i=1}^{n_i}\|  W_{o, i} - \sum\limits_{k=1}^K P^k_{o, i}D^k_{i}\|_F^2=\sum\limits_{i=1}^{n_i}\| \widetilde{W}_{:, i} - \sum\limits_{k=1}^K \widetilde{P^k_{:, i}}\widetilde{D^k_i}^\intercal\|_F^2,
	\end{equation*}
	where $\widetilde{W}_{:, i}$ = $\widetilde{\mathbf{W}}[:,i,:]$ is a slice of the reshaped tensor $\widetilde{\mathbf{W}}\in$ $\mathbb{R}^{n_o\times n_i\times (k_hk_w)}$ from $\mathbf{W}$, $\widetilde{D^k_i}$ = $Vec(D^k_i)$ is the vector view of $D^k_i$, and $\widetilde{P^k_{:, i}}$ = $\mathbf{P}^k[:, i, :, :]$ is a tensor fiber.
	$\widetilde{W}_{:, i}$ can be viewed as a matrix of size ${n_o\times (k_hk_w)}$ with $rank(\widetilde{W}_{:, i}) \le \min\{n_o, k_hk_w\}$.
	
	Suppose $\widetilde{W}_{:, i}=USV^\intercal$ is the singular value decomposition. Let $\widetilde{P^k_{:, i}}=U_kS(k)$ where $U_k$ is the $k$-{th} column of $U$ and $S(k)$ is the $k$-th singular value, $\widetilde{D^k_i} = V_k$ where $V_k$ is the $k$-{th} column of $V$.
	When we set $K =\max_i rank(\widetilde{W}_{:, i})$, the rightmost term equals to zero. And then $\|  W_{o, i} - \sum_{k=1}^K P^k_{o, i}D^k_{i}\|_F^2=0$.
Hence $\mathbf{W} = \sum_{k=1}^K \mathbf{P}^k\circ \mathbf{D}^k$.
	As $n_o \gg k_hk_w$ generally holds, $K = \max_i rank(\widetilde{W}_{:, i}) \le \min\{n_o, k_hk_w\}$.
	Therefore, $K\le k_hk_w$ holds.
	
	The computational complexity of this expansion is thus $O(K\times H\times W\times(n_i\times k_h k_w + n_i\times n_o))$, where $H\times W$ is the resolution of current feature map.
	The computing cost ratio of the regular convolution to this expansion is $r = \nicefrac{k_hk_w}{K(k_hk_w/n_o + 1)}$. As $n_o \gg k_hk_w$, $r \approx \nicefrac{k_h k_w}{K}$. Since $K\le k_hk_w$, $r \geq 1$ holds.
	That means the lossless expansion does not increase the computational complexity over the regular convolution.
\end{proof}
	
	\iffalse
	\textbf{PW+DW}: This is similar to the DW+PW case. $D^k$ is of size $n_o\times k_h \times k_w$, and $P^k$ is of size $n_o\times n_i\times 1\times 1$.
	The response difference is:
	%\begin{footnotesize}
		\begin{equation*}
        \small
		\begin{aligned}
		\|y-y''\|_2^2 &= \|(W - \sum\nolimits_{k=1}^K D^k\circ P^k)*x \|_2^2 \le \sum_{o=1}^{n_o}\sum_{i=1}^{n_i}\|  W_{o, i} - \sum\nolimits_{k=1}^K P^k_{o, i}D^k_{o}\|_F^2\cdot \|x_i\|_F^2
		\end{aligned}
		\end{equation*}
	%\end{footnotesize}
	where $D^k_o = D^k[o, :, :]$. And for the right-hand side,
	\begin{equation*}\label{eqnpwdwsvd}
	\sum\limits_{o=1}^{n_o}\sum\limits_{i=1}^{n_i}\|  W_{o, i} - \sum\limits_{k=1}^K P^k_{o, i}D^k_{o}\|_F^2=\sum\limits_{o=1}^{n_o} \|  \widetilde{W_{o,:}} - \sum\limits_{k=1}^K \widetilde{P^k_{o,:}}\widetilde{D^k_{o}}^\intercal\|_F^2
	\end{equation*}
	where $\widetilde{W}_{o,:}=\widetilde{W}[o, :,:]$, $\widetilde{D^k_o}=Vec(D^k_o)$, and $\widetilde{P^k_{o,:}}=P^k[o, :, :, :]$.
	Similarly, there exists $\{P^k, D^k\}_{k=1}^K$, $K\le k_hk_w$ such that $W = \sum_{k=1}^K D^k\circ P^k$.
	Also, the computation cost does not increase over regular convolutions.
	\fi

\begin{figure}[]	
	\begin{minipage}{0.33\linewidth}
		\centering
		\includegraphics[width=0.95\textwidth]{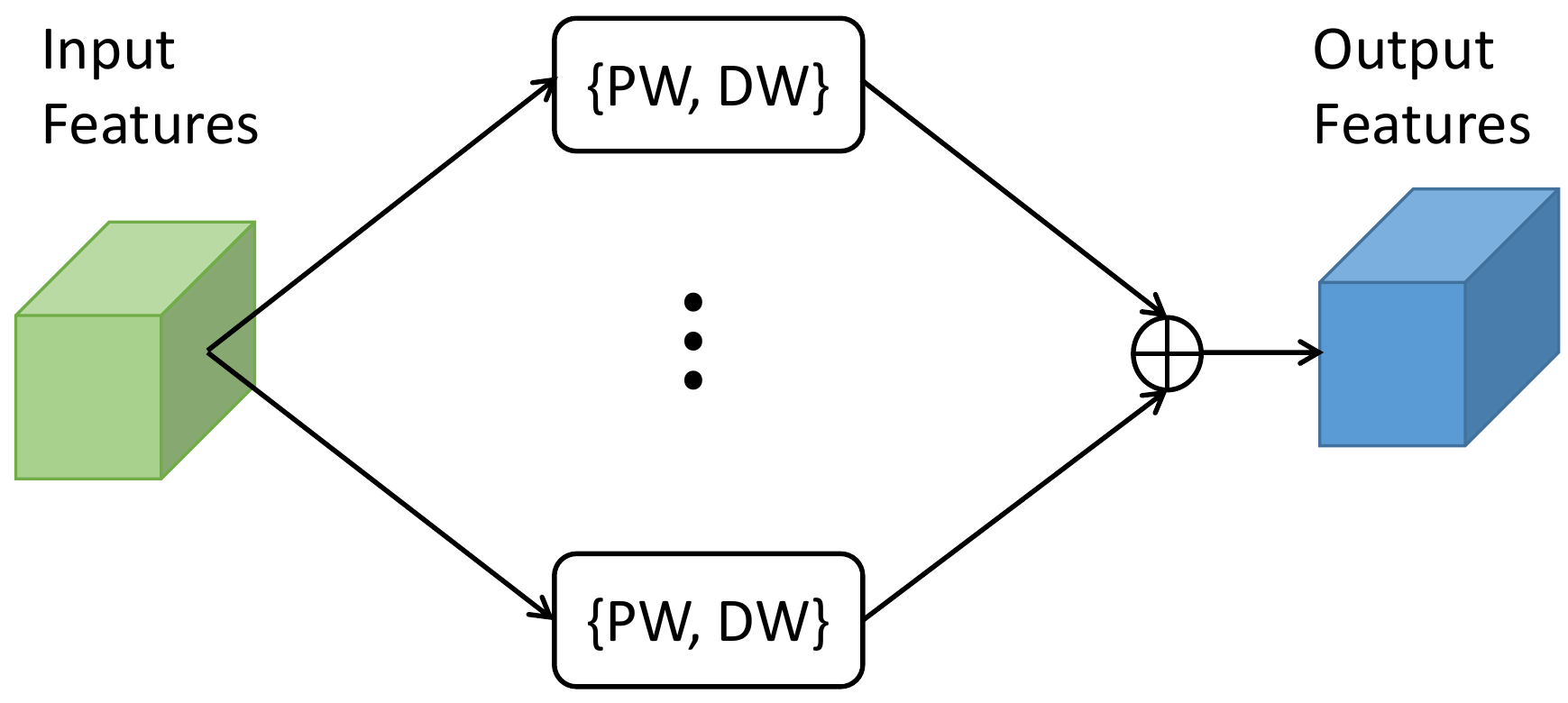}
        \vspace{1ex}
		\caption{Approximating regular convolution by the sum of depthwise separable convolutions. \{PW, DW\} can be either PW+DW or DW+PW.}
		\label{fig.decoupling}
	\end{minipage}
	\hspace{1ex}
	\begin{minipage}{0.62\linewidth}
		\includegraphics[width=0.47\textwidth]{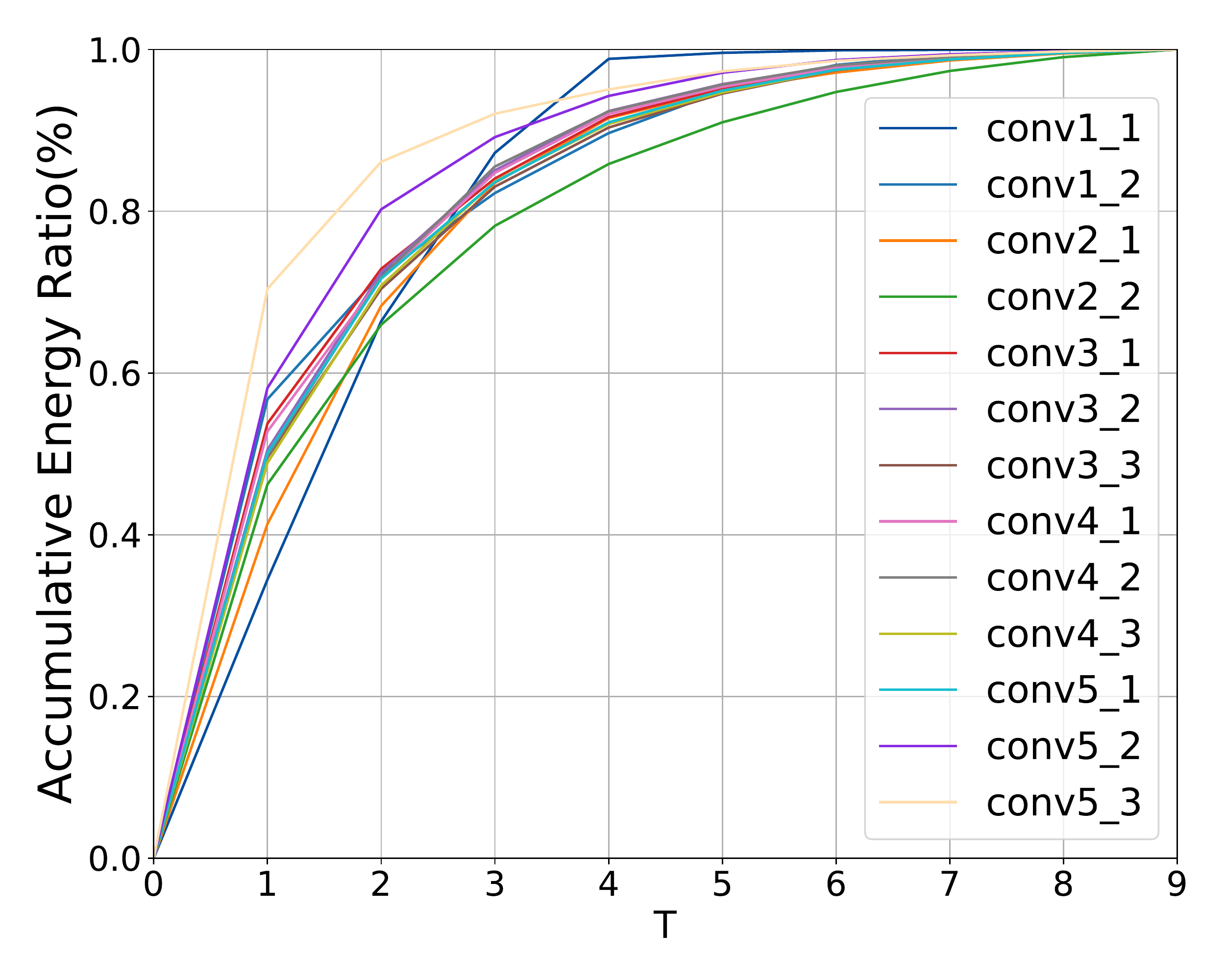}
		\includegraphics[width=0.47\textwidth]{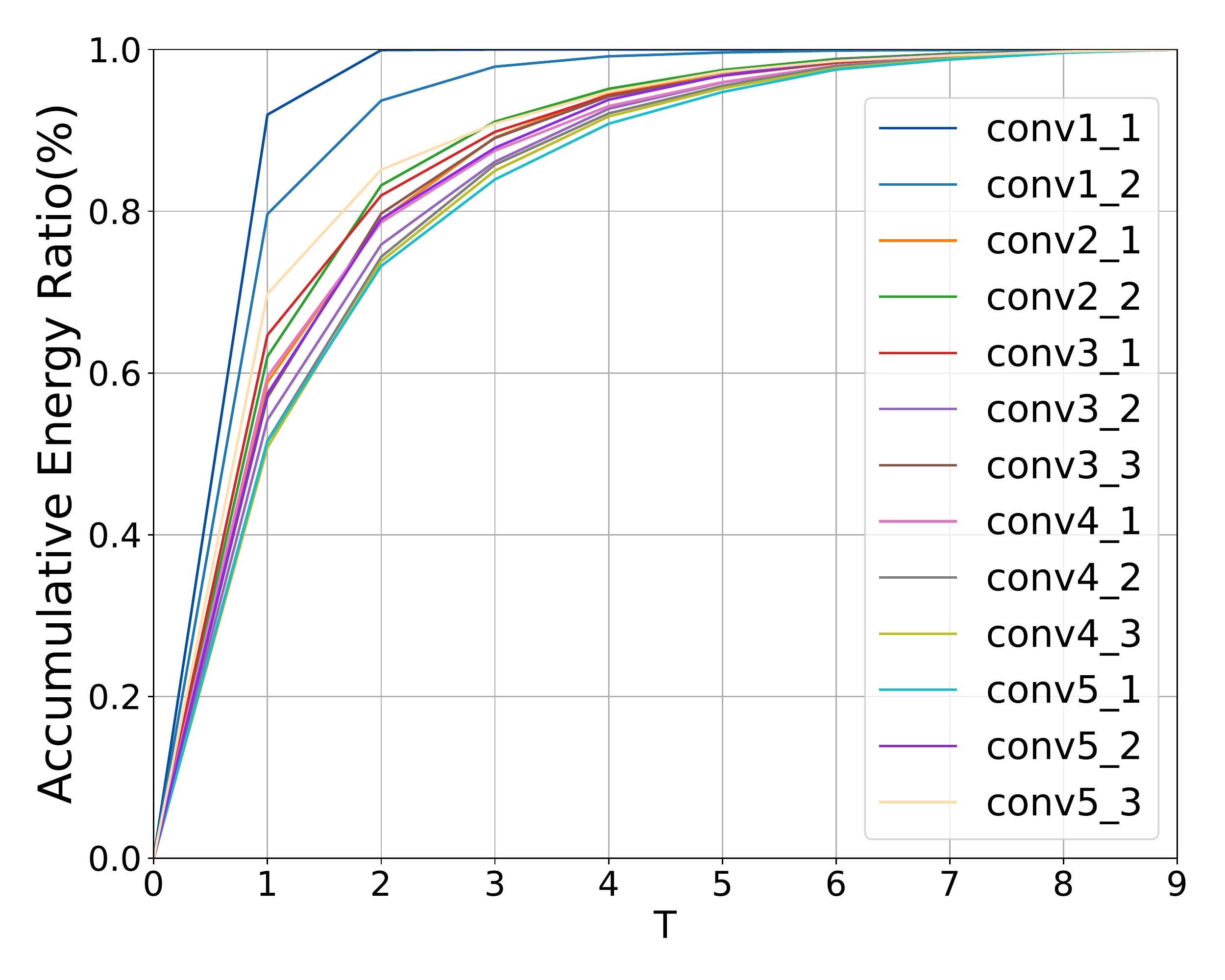}
        \vspace{1ex}
		\caption{Accumulative energy ratio in the DW+PW case (left) and in the PW+DW case (right).}
		\label{fig.svdenergy}
	\end{minipage}
	\vspace{-0.2in}
\end{figure}
\vspace{-2ex}
\section{Network Decoupling}
Theorem \ref{th1} actually presents a closed-form tensor decomposition to decouple regular convolutions into depthwise separable convolutions.
We name this solution as \textit{exact network decoupling}, and $K$ as the \textit{decoupling rank} of $\mathbf{W}$ which reflects the coupling induced redundancy in $\mathbf{W}$. When $K$ is low, for instance, $K=1$, there is significant redundancy in $\mathbf{W}$, so that \textit{exact network decoupling} can bring great computation cost reduction.
When $K=k_hk_w$, there is no redundancy, and hence no benefit with \textit{exact network decoupling}.

\vspace{-1ex}
\subsection{Approximated Network Decoupling}
\label{sec.decoupling}
For less redundant CNN layers, the \textit{exact network decoupling} may bring unsatisfied speedup.
Due to the decomposition nature, the energy is not equally distributed among the $K$ depthwise separable convolution (DSC) blocks.
In fact, substantial energy is concentrated in a fraction of those $K$ DSC blocks. % is associated with a singular value to represent its energy magnitude.
Hence, we can realize the \textit{approximated network decoupling} for better CNN speedup based on the following corollary deduced from Theorem \ref{th1}.
\vspace{-0.1in}
\begin{corollary}\label{th.speedup}
	Given a regular convolution tensor $\mathbf{W}$ with spatial kernel size $k_h\times k_w$, we can approximate it with top-$T$ ($\le K$) depthwise separable convolutions as
	\begin{equation}\label{eq:approx}
    \small
    \setlength{\abovedisplayskip}{1ex}
    \setlength{\belowdisplayskip}{1ex}
	\mathbf{W} \approx \left\{ \begin{array}{ll}
	\sum\nolimits_{k=1}^T \mathbf{P}^k\circ \mathbf{D}^k  & \textrm{for~DW+PW}\\
	\sum\nolimits_{k=1}^T \mathbf{D}^k\circ \mathbf{P}^k  & \textrm{for~PW+DW}, \\
	\end{array}
	\right .
	\end{equation}
	and the acceleration over the original regular convolution is $k_hk_w/T$.
\end{corollary}
\vspace{-0.1in}

Figure \ref{fig.decoupling} illustrates our depthwise separable convolution approximation to the regular convolution.
Note that the proposed network decoupling does not require any training data.
We will further study the possibility to combine it with other existing training-free methods for even larger CNN speedup in Section \ref{sec.comb}.

How good is the approximated decoupling?
Let's take VGG16 for example, in which $k_hk_w=9$ for all the convolutional layers.
For the DW+PW case, we compute the singular values of $\widetilde{W}_{:, i}$ ($i\in[n_i]$), and average the ratio of the square sum of the top-$T$ largest singular values to the total square sum. The same thing is done for PW+DW.
Figure \ref{fig.svdenergy} plots the average energy ratio for both cases.
We can see that substantial energy is from several top singular vectors.
For example, in both cases, the top-4 singular vectors contribute over $90\%$ energy in all the layers except the conv2\_2 layer in DW+PW case.
Especially, in the conv1\_2 layer by PW+DW, the top-$4$ singular vectors account for over $99\%$ energy.
This indicates that we can only use a fraction of depthwise separable convolutions to precisely approximate the original regular convolutions.
Note that energy in PW+DW is more concentrated than that in DW+PW, which means PW+DW can realize the same quality approximation with fewer DSC blocks, and thus yields a better speedup. 
One possible reason is that for the layer with tensor kernel $n_i \times k_w \times k_h \times n_o$, 
the PW+DW case will produce $n_o$ separate DW channels, while the DW+PW case only has $n_i$ DW channels. 
When $n_o > n_i$, the PW+DW case will have more parameters than the DW+PW case so that PW+DW may have better approximation.
We will verify this result by experiments later.

\subsection{Complementary Methods}
\label{sec.comb}
ND focuses on decomposing a regular convolution into a sum of $T$ depthwise separable convolutions.
The method not only has a closed-form solution, but also is training-data free.

Besides our work, there are some works considering network decomposition from different perspectives, like
channel decomposition \cite{Zhang2016Accelerating},
spatial decomposition \cite{Jaderberg2014Speeding}, and
channel pruning \cite{He2017Channel}.
All these methods require a small calibration dataset (for instance 5000  images from 1.2 million ImageNet images for ImageNet models) to reduce possible accuracy loss.
Different from these methods, the proposed network decoupling does not involve any channel reduction and spatial size reduction, which implies our method should be complementary to them.
Hence we propose to combine network decoupling with these methods to further accelerate deep CNN models.
Additionally, this combination even provides us the possibility to reduce the accumulated error with the calibration set.

Let's take the combination of ND and channel decomposition as an example. 
For a layer with tensor kernel $n_i \times k_w \times k_h \times n_o$, we first apply CD to decompose it into two layers, where the first layer has tensor kernel 
$n_i \times k_w \times k_h \times d$, and the second layer is $1\times 1$ conv-layer with tensor kernel $d \times 1 \times 1 \times n_o$. 
CD sets $d < n_o$ to ensure acceleration. 
We then apply ND to the layer with kernel $n_i \times k_w \times k_h \times d$ to decouple it into one point-wise convolution layer and one depthwise convolution layer.
We process the next layer in the original network with the same procedure. 
As is known, CD will minimize the reconstruction error of the responses between original networks and decomposed one with a small calibration set. 
When apply CD to next layer, it will compensate the accumulated error somewhat from two successive approximations in previous layer. 
This procedure is adopted sequentially until all the layers in the original network are processed. 
We will show in the experiments that the combined solution can bring significantly better CNN model acceleration than any single methods.

\vspace{-1ex}
\section{Experiments}
We conduct extensive experiments to evaluate the proposed method on different network structures such as VGG16 \cite{simonyan2014very} and ResNet18 \cite{he2016deep} pre-trained on ImageNet \cite{imagenet_cvpr09} with Caffe \cite{jia2014caffe}.
%Note that all the results here are measured by single center-crop test on the ILSVRC 2012 validation set.
The top-5 accuracy (measured by single center-crop test) of VGG16 is $88.66\%$ with $15.35G$ FLOPs, and ResNet18 is $88.09\%$ with $1.83G$ FLOPs.
We also evaluate the proposed method for the object detection framework SSD300 \cite{liu2016ssd} (with VGG16 as backbone) on PASCAL VOC 2007 benchmark \cite{everingham2010pascal}. We further study how fine-tuning can help extremely decoupled networks.
All these studies are conducted \textbf{without} fine-tuning unless specified.

\begin{figure*}[]
	\centering	
	\begin{minipage}{0.30\linewidth}
		\includegraphics[width=\textwidth]{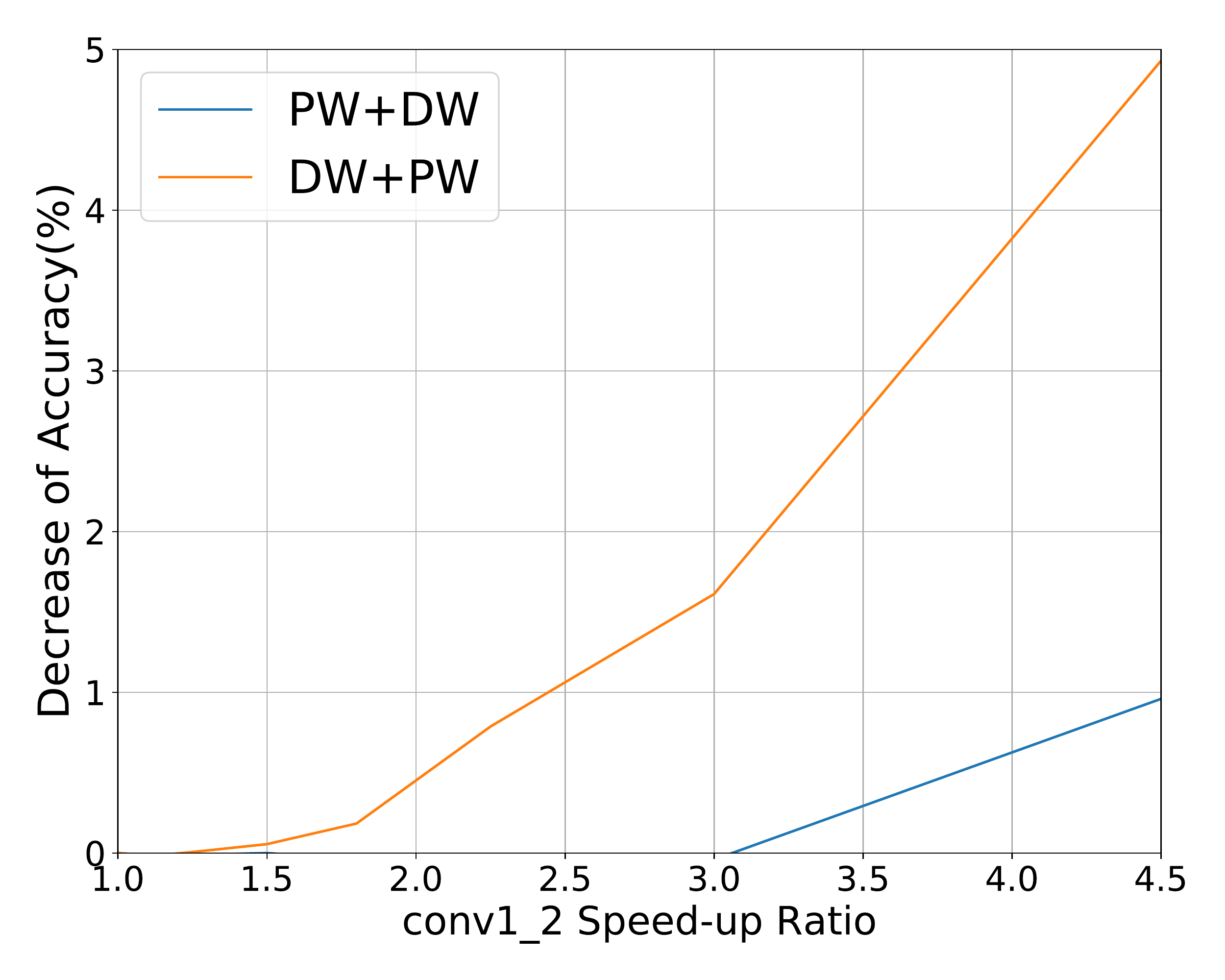}
	\end{minipage}
	\hspace{1ex}
	\begin{minipage}{0.30\linewidth}
		\includegraphics[width=\textwidth]{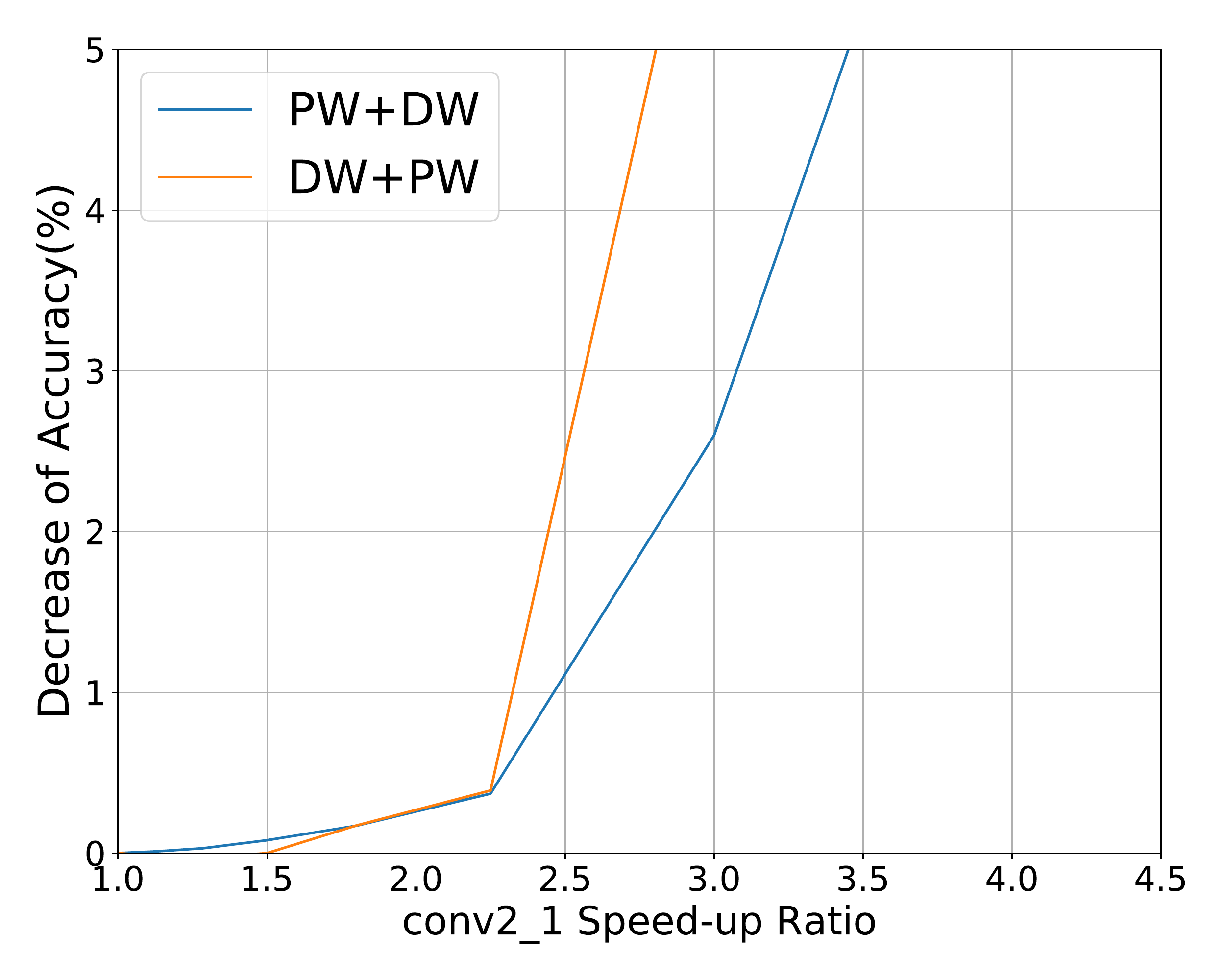}
	\end{minipage}
	\hspace{1ex}
	\begin{minipage}{0.30\linewidth}
		\includegraphics[width=\textwidth]{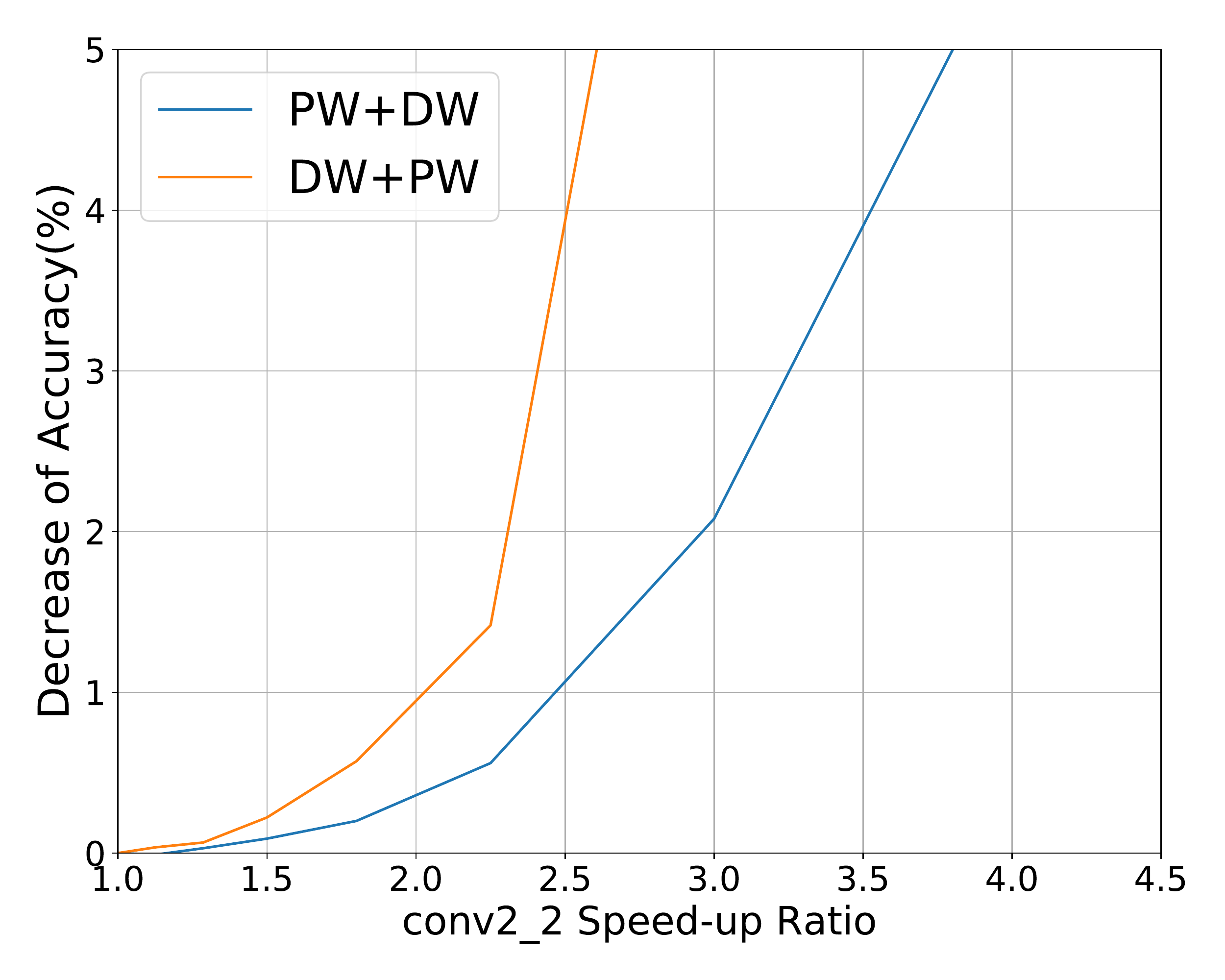}
	\end{minipage}
	\begin{minipage}{0.30\linewidth}
		\includegraphics[width=\textwidth]{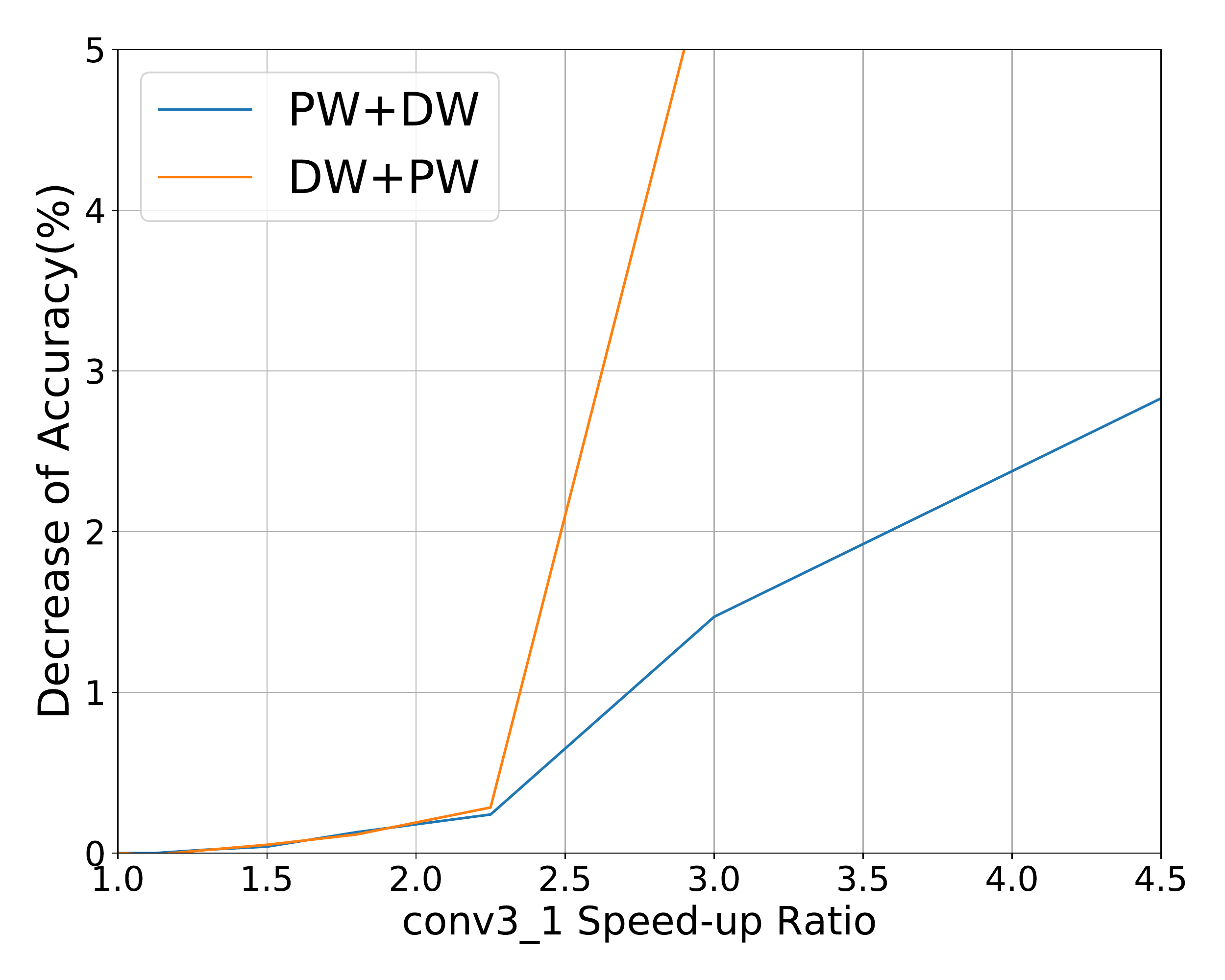}
	\end{minipage}
	\hspace{1ex}
	\begin{minipage}{0.30\linewidth}
		\includegraphics[width=\textwidth]{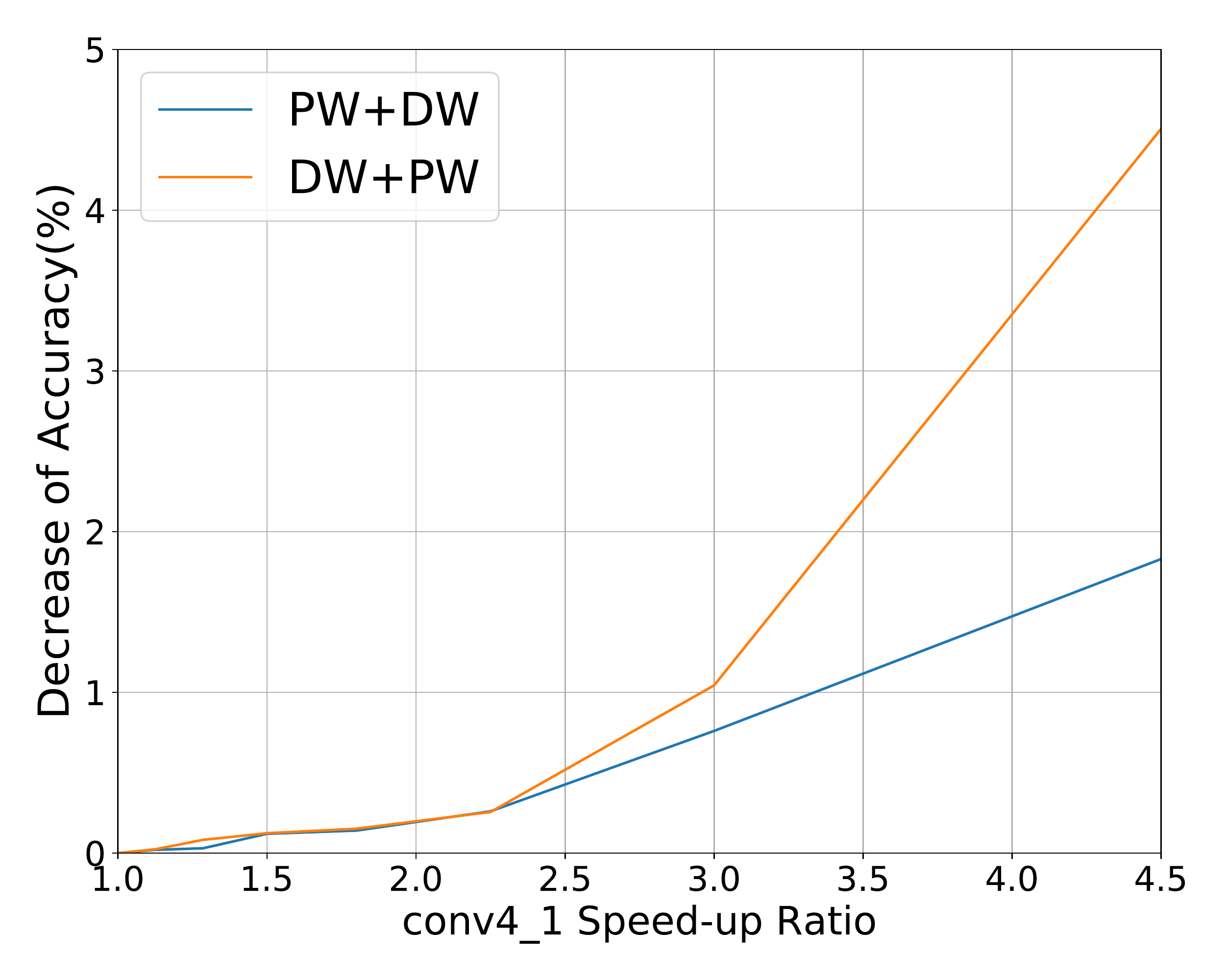}
	\end{minipage}
	\hspace{1ex}
	\begin{minipage}{0.30\linewidth}
		\includegraphics[width=\textwidth]{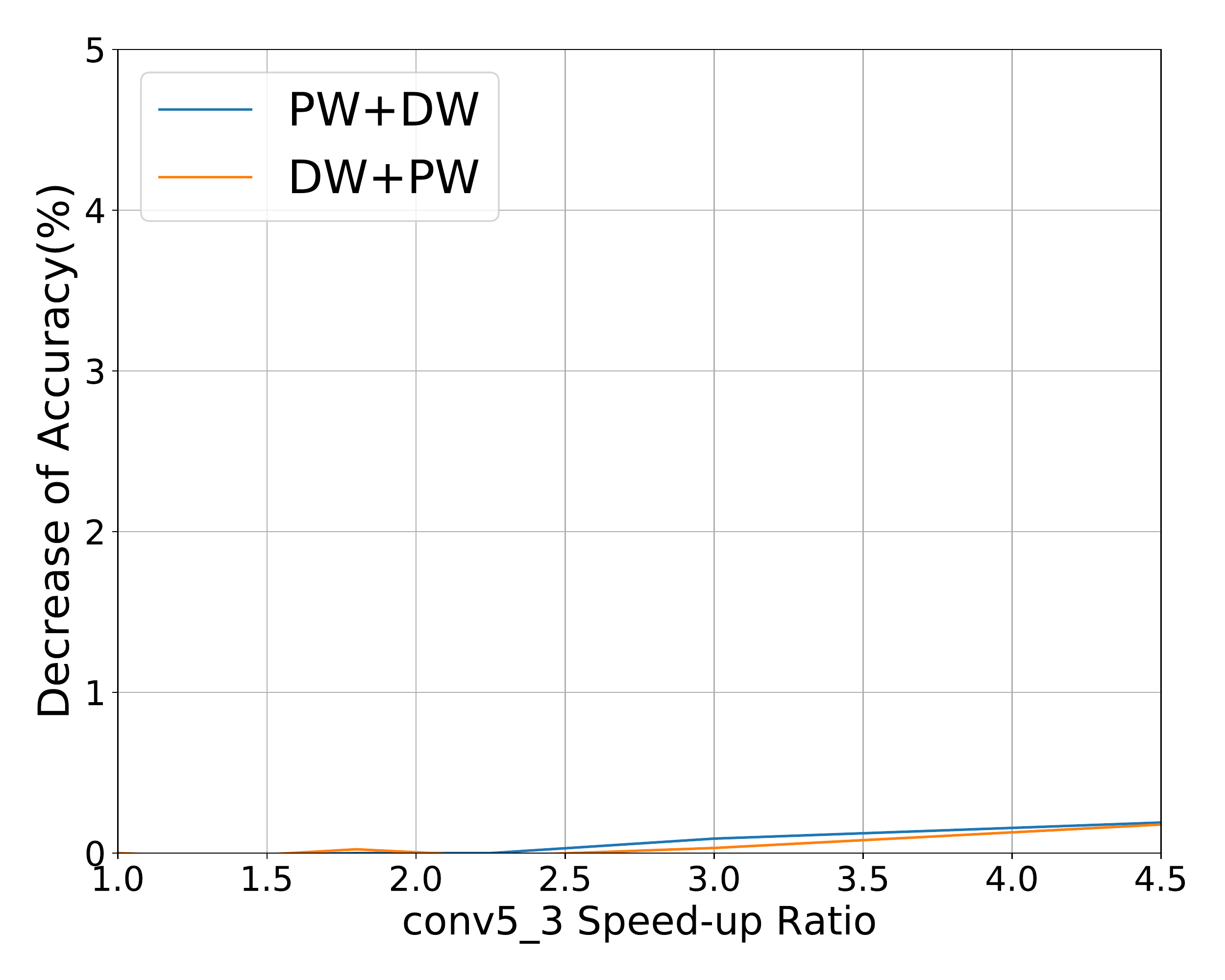}
	\end{minipage}
	\vspace{1ex}
	\caption{\textbf{DW+PW vs. PW+DW} decoupling for VGG16: single layer performance under different speedup ratios, measured by decrease of top-5 accuracy on ImageNet (\textit{smaller is better}). The speedup ratios are computed by the theoretical complexity of that layer.}
	\label{fig.single}
	\vspace{-0.2in}
\end{figure*}

\vspace{-1ex}
\subsection{Single Layer Decoupling}
\label{sec.single}
We first evaluate the single layer acceleration performance using our network decoupling method.
In this study, we decouple one given layer with all the remaining layers unchanged.
The speedup ratio reported only involves that single layer, which is shown as the theoretical ratio computed by the complexity (FLOPs).
Figure \ref{fig.single} illustrates the speedup vs. the accuracy drop for different layers of VGG16 under both the DW+PW and the PW+DW decoupling.

We can see that when speeding up a single layer by $2\times$, the accuracy drop is rather marginal or negligible, especially in the PW+DW case.
In this case, there is no accuracy drop for layer conv1\_2 and conv5\_3, and less than $0.5\%$ accuracy loss for all other layers.
It also shows that PW+DW decoupling consistently outperforms the DW+PW case.
This result can be explained by Figure \ref{fig.svdenergy}: to accumulate the same energy, PW+DW decoupling needs smaller $T$, hence better acceleration.
In the later experiments, we only use PW+DW decoupling.
We also find that decoupling brings less speedup for the intermediate layers compared with shallower and deeper (front and end) layers. It implies shallower and deeper layers have much more redundancy.
This property is different from channel decomposition \cite{Zhang2016Accelerating} and channel pruning \cite{He2017Channel}, where redundancy concentrates only in the shallower layers. This verifies network decoupling and channel decomposition/pruning are intrinsically complementary, and we can combine our decoupling with them to speed up CNNs further.

\begin{table}[]
	\centering
	\footnotesize
	\setlength{\abovecaptionskip}{10pt}
	\setlength{\belowcaptionskip}{0pt}
	\vspace{-7pt}
	\subfloat[Single methods]{
		\label{tab:acc result single}%label for 1st subtable
	
		\begin{minipage}{0.4\linewidth}
			\vspace{19pt}
			\centering
			\begin{tabular}[tb]{|c|c|c|}
				\hline
				Method & FLOPs & top-1 drop (\%)\\ \hline
				Original VGG16 & $15.35G$ & 0 \\
				\hline
				CD\cite{Zhang2016Accelerating} & $6.52G$ & 2.10\\
				\hline
				SD\cite{Jaderberg2014Speeding} & $7.20G$ & 1.96\\
				\hline
				CP\cite{He2017Channel} & $9.89G$ & 1.68\\ \hline
				Ours & $8.61G$ & 1.55 \\
				\hline
			\end{tabular}
		\end{minipage}%
	}% end of subfloat
    \hfill 
	\vspace{-10pt}
	\subfloat[Combined methods]{
		\label{tab:acc result combined}%label for 2nd subtable
		\vspace{-10pt}
		\begin{minipage}{0.55\linewidth}
			\centering
			\begin{tabular}[hbt]{|c|c|c|}
				\hline
				\multirow{2}{*}{Method} & \multicolumn{2}{|c|}{FLOPs} \\
				\cline{2-3}
				& without ND & with ND (ND+X)\\
				\hline
				CD~\cite{Zhang2016Accelerating} & $6.52G$ & $4.72G$ \\
				\hline
				SD~\cite{Jaderberg2014Speeding} &$7.20G$ &$\textbf{4.15G}$\\
				\hline
				CP~\cite{He2017Channel} & $9.89G$ & $8.49G$ \\
				\hline
				CD+SD~\cite{Zhang2016Accelerating} &$4.32G$ &$4.28G$\\
				\hline
				CD+CP & $7.16G$ & $5.45G$ \\
				\hline
				CD+SD+CP~\cite{He2017Channel} & $4.92G$ & $4.70G$ \\ \hline
			\end{tabular}
		\end{minipage}%
	}% end of subfloat
	\caption{Acceleration performance on VGG16 with (a) single method and (b) combined method. ``with ND" means the experiments are conducted in combination with our network decoupling (ND). The number of FLOPs is computed by the theoretical complexity of the approximated model. 
Except for original VGG16, all the other models are tuned with fixed 1.0\% top-5 accuracy drop. (a) also lists the corresponding top-1 accuracy drops.}\label{whatever}
	\vspace{-0.2in}
\end{table}
\vspace{-1ex}
\subsection{Whole Model Decoupling of VGG16}
\subsubsection{Experiments on Single Methods}\label{sec:single method}
\vspace{-1ex}
Next, we evaluate the acceleration performance of PW+DW decoupling on the whole VGG16 model. We sequentially approximate the layers involved (conv1\_2 to conv5\_3).
The first layer conv1\_1 is not decoupled, since it only contributes $0.3\%$ computation. Guided by the single layer experiments above, we decouple more aggressively for both shallower layer and deeper layer, by using smaller $T$.
The computing cost (in FLOPs) of the decoupled model is reported with fixed 1\% top-5 accuracy drop, where the corresponding top-1 accuracy drops range from 1.5\% to 2.1\%.
As shown in Table \ref{tab:acc result single},  our decoupling method can reduce about $45\%$ of the total FLOPs, which demonstrates significant redundancy inside VGG16.

We also compare our decoupling method with other design-time network optimization methods like channel decomposition (CD) \cite{Zhang2016Accelerating},  spatial decomposition (SD) \cite{Jaderberg2014Speeding} and channel pruning (CP) \cite{He2017Channel}, which are recent state-of-the-art training/fine-tuning free solutions for CNNs acceleration.
%We re-implement these methods by ourselves so that our implementation could repeat their reported results.
The comparison is based on the re-implementation of \cite{Zhang2016Accelerating,He2017Channel}.
Note that all these three compared methods require a small portion of training set (i.e., 5000 images) in their optimization procedure.
Even though, our method still outperforms the data-driven based channel pruning as shown in Table \ref{tab:acc result single}, which indicates that decoupling convolutions is promising for CNN acceleration.
Although our method performs somewhat worse than CD and SD, we should emphasize that our method is totally data-free, while all these three are data-driven methods.
Moreover, we will show below that our network decoupling combined with these methods will bring state-of-the-art CNN acceleration in the training-free setting.

\vspace{-1ex}
\subsubsection{Experiments on Combined Methods}\label{sec:combined method}
\vspace{-1ex}
As discussed in Section \ref{sec.comb} and \ref{sec.single}, our network decoupling is intrinsically complementary to channel decomposition, spatial decomposition and channel pruning.
In this part, we will study the performance of combined methods.
We not only test the performance of combining our decoupling with each of the above methods separately, but also conduct the experiments of the existing compound methods with/without our decoupling scheme.
Note that these combined models require a small portion of training dataset for data-driven based optimization.
For fair comparison, we randomly pick up $5000$ images out of ImageNet training set (1.2 million images) and use them for all the evaluated methods.

Table \ref{tab:acc result combined} shows the results. Clearly, combined with our network decoupling, each of the above methods has significant improvement, which verifies that our decoupling exploits a new cardinality.
In the best result, our combined method (ND+SD) could reduce about $73\%$ of the original FLOPs ($3.7\times$ speedup), with only 1\% top-5 accuracy drop.

Interestingly, we find that CD+CP performs worse than CD alone (FLOPs increase from $6.52G$ to $7.16G$). We speculate that channel pruning and channel decomposition are not complementary to some extent. They both exploit the inter-channel redundancy and reduce the number of channels.
Furthermore, both CD and CP captures more redundancy in the shallower layers according to \cite{Zhang2016Accelerating,He2017Channel}.
Therefore, their combination may yield conflicts between them.
Compared with CD and CP, our network decoupling can capture redundancy from both shallower layers and deeper layers (see Figure \ref{fig.single}).
Hence, we can enhance both of them. This result also holds for the combination with spatial decomposition.

\begin{table}[]
	\centering
	\footnotesize
	\setlength{\abovecaptionskip}{10pt}
	\setlength{\belowcaptionskip}{0pt}
	\begin{minipage}{0.51\linewidth}
		\centering
		\begin{tabular}[h]{|c|c|c|}
			\hline
			Method & top-5 Accuracy & FLOPs \\ \hline
			Original ResNet18 & $88.09$ & $1.83G$ \\ \hline
			CD~\cite{Zhang2016Accelerating} & $83.69$ & $1.32G$ \\ \hline
			SD~\cite{Jaderberg2014Speeding} & $85.20$ & $1.25G$\\ \hline
			Ours & $86.68$ & $1.20G$ \\ \hline
		\end{tabular}
		\vspace{-0.1in}
		\caption{Training-free acceleration for ResNet18.}\label{tab:resnet}
	\end{minipage}
	\begin{minipage}{0.48\linewidth}
		\centering
		\begin{tabular}{|c|c|c|}
			\hline
			Method & mAP & FLOPs \\ \hline
			Original SSD300 & $82.29$ & $31.37G$ \\ \hline
			CD~\cite{Zhang2016Accelerating} & $80.99$ & $18.69G$\\ \hline
			Ours & $81.41$ & $20.27G$ \\ \hline
			Ours+CD & $80.71$ & $15.01G$ \\ \hline
		\end{tabular}
		\vspace{-0.1in}
		\caption{Training-free acceleration for SSD300.}\label{tab:ssd}
	\end{minipage}
	\vspace{-0.25in}
\end{table}

\vspace{-1ex}
\subsection{Decoupling ResNet}
Modern networks like ResNet \cite{he2016deep} are designed for both efficiency and high accuracy, which are usually not easy to accelerate in the training-free setting.
In this part,
we evaluate network decoupling on ResNet18, which has a VGG16 comparable top-5 accuracy ($88.09$) but with much lower computation cost ($1.83G$ FLOPs).

Table \ref{tab:resnet} shows that our network decoupling alone can reduce about $34\%$ of total FLOPs ($1.5\times$ speedup) with $1.4\%$ drop of top-5 accuracy. This result is not as graceful as that of VGG16, since modern structures tend to have less redundancy by design.
As a comparison, CD is not a data-free solution, which only reduces 28\% total FLOPS (1.36$\times$ speedup), and brings 4.4\% top-5 accuracy drop.
Other methods like \cite{He2017Channel} handle the accuracy drop with a limited epoch fine-tuning.
We will explore the benefit of this solution later.

\vspace{-1ex}
\subsection{Decoupling SSD300}
Object detection suffers from even higher computing cost than image classification due to its relatively high-resolution input.
We evaluate our network decoupling for one of widely used object detection framework SSD300 \cite{liu2016ssd} on the PASCAL VOC 2007 \cite{everingham2010pascal}.
The backbone of SSD300 is based on VGG16, which is popular in many object detection frameworks \cite{Ren2017Faster}.
The performance is measured by mean Average Precision (mAP) and total FLOPs of the detection networks.

Different from \cite{Zhang2016Accelerating,He2017Channel}, we extract backbone from the pre-trained SSD model, decompose the filters inside the backbone, and then use the approximated backbone to take the detection task, where no fine-tuning is involved. From Table \ref{tab:ssd}, we observe that network decoupling alone achieves $35\%$ FLOPs reduction with mAP drop less than $1.0\%$, which is acceptable in most scenes.
Further, if combining our approach with other training-free methods like channel decomposition, we could at most reduce the FLOPs to $48\%$ of the original model ($2.1\times$ speedup) with only $1.58\%$ loss in mAP, which is also acceptable.
Note although the backbone is based on VGG16, SSD300 has different speedup performance from that of the VGG16 classification network. This is because SSD300 changes the model parameters of backbone network during training procedure due to different target loss functions and different inputs \& resolutions, 
so that detection backbone has less redundance than that of classification network.
Similar phenomena have been observed by \cite{Zhang2016Accelerating,He2017Channel}.

\vspace{-1ex}
\subsection{Extremely decoupling with Fine-tuning}
Here we study the possibility of combining network decoupling with fine-tuning for even better speedup like \cite{kim2015compression,He2017Channel}.
We first perform extremely or aggressively network decoupling, which will usually bring large accuracy drop.
We then fine-tune the extremely decoupled models with initial learning rate 0.0001, and decrease the learning rate 1/10 every 4 epochs. The results are shown in Table \ref{tab:finetuning}.

We aggressively decouple VGG16 with $T = 2$, which yields a model with top-5 accuracy 19.4\%.
We recover the top-5 accuracy of this model to 88.1\% (-0.56\% to baseline) with just 10 epochs of fine-tuning, yielding 3.9$\times$ speedup (3.96G FLOPs).
We aggressively decouple ResNet18 with $T = 3$, which yields a model with top-5 accuracy 29.2\%.
We recover the top-5 accuracy to 88.22\% (+0.13\% to baseline) with just 6 epochs of fine-tuning, yielding 2.1$\times$ speedup (0.89G FLOPs).
Note that the number of epochs used here is less than $1/10$ of the training from scratch solutions.
In comparison, the fine-tuning-free network decoupling only provides 1.8$\times$ speedup for VGG16 with top-5 accuracy 87.66\%,
and 1.5$\times$ speedup for ResNet18 with top-5 accuracy only 86.68\%.
It is obvious that extremely decoupling plus fine-tuning provides not only better speedup but also much higher accuracy.

Besides, we compare our results with the state-of-the-art training-time network optimization methods (with fine-tuning on VGG16).
ThiNet \cite{luo2017thinet} is a pure fine-tuning based network optimization method, which requires 32 epochs of fine-tuning to obtain $3.3\times$ speedup with $0.52\%$ top-5 accuracy drop.
It shows that our method achieves better speedup over ThiNet (3.9 vs. 3.3), while with much fewer fine-tuning epochs (10 vs. 32).
Channel pruning(+fine-tuning) \cite{He2017Channel} requires 10 epochs of fine-tuning to obtain about 4$\times$ speedup with 1.0\% top-5 accuracy drop.
It shows that our method achieves similar speedup (3.9 vs. 4.0), while with less accuracy drop (0.52\% vs. 1.0\%).% and fewer fine-tuning epochs (10 vs 20).
\begin{table}[]
	\centering
	\footnotesize
	\begin{tabular}[h]{|c|c|c|c|}
		\hline
		Model &Epochs & top-5 Accuracy & Speedup \\ \hline
        \hline
		Original VGG16 & 100 & $88.66$ & $1.00$ \\ \hline
        ND (auto-tuned T)  & 0 & $87.66 (-1.00)$ & $1.8\times$ \\ \hline
        ND(T=2)+VGG16& 10  & $88.10 (-0.56)$ & $3.9\times$ \\ \hline
        ThiNet \cite{luo2017thinet} &  32 & $88.14(-0.52)$ & 3.3$\times$\\ \hline
        CP+finetune \cite{He2017Channel}  & 10 & $87.66(-1.00)$& 4.0$\times$\\ \hline
		\hline 
		Original ResNet18 & 100 & $88.09$ & $1.00$ \\ \hline
        ND (auto-tuned T)& 0 & $86.68(-1.41)$ & $1.5\times$ \\ \hline
		ND(T=3)+ResNet18 & 6 & $88.22 (+0.13)$ & $2.0\times$ \\ \hline
	\end{tabular}
	\vspace{0.1in}
	\caption{Results of fine-tuning of extremely decoupled VGG16 (top part) and ResNet18 (bottom part). Digits in the bracket show accuracy change compared with original models.
		``ND+" means we apply our method to aggressively decouple the original model and fine-tune it for given epochs. On VGG16, we also list results by other two fine-tuning based network optimization methods as comparison. }\label{tab:finetuning}
	\vspace{-0.2in}
\end{table}

\vspace{-2ex}
\section{Conclusion}
This paper analyzes the mathematical relationship between regular convolutions and depthwise separable convolutions,
and proves that the former one can be approximated with the latter one precisely.
We name the solution \textit{network decoupling} (ND), and demonstrate its effectiveness on VGG16, ResNet as well as object detection network SSD300.
We further show that ND is complementary to existing training-free methods,
and can be combined with them for splendid acceleration.
ND could be an indispensable module for deploy-time network optimization, as well as provides theoretic supports for possible future studies.

\vspace{1ex}
\noindent\textbf{Acknowledgement:}
Jianbo Guo is supported in part by the National Basic Research Program of China Grant 2015CB358700, the NSFC Grant 61772297, 61632016, 61761146003. Weiyao Lin is supported in part by the NSFC Grant 61471235 and the Shanghai "The Belt and Road" Young Scholar Grant (17510740100). Jianguo Li is the corresponding author. 

{\small
\bibliography{bmvc_final}

\begin{thebibliography}{33}
\providecommand{\natexlab}[1]{#1}
\providecommand{\url}[1]{\texttt{#1}}
\expandafter\ifx\csname urlstyle\endcsname\relax
  \providecommand{\doi}[1]{doi: #1}\else
  \providecommand{\doi}{doi: \begingroup \urlstyle{rm}\Url}\fi

\bibitem[Chen et~al.(2015)Chen, Wilson, et~al.]{hashnet}
Wenlin Chen, James Wilson, et~al.
\newblock Compressing neural networks with the hashing trick.
\newblock In \emph{ICML}, 2015.

\bibitem[Chollet(2016)]{Chollet2016Xception}
Fran{\c{c}}ois Chollet.
\newblock Xception: Deep learning with depthwise separable convolutions.
\newblock \emph{arXiv preprint arXiv:1610.02357}, 2016.

\bibitem[Courbariaux and Bengio(2016)]{Courbariaux2016BinaryNet}
M.~Courbariaux and Y.~Bengio.
\newblock Binarynet: Training deep neural networks with weights and activations
  constrained to +1 or -1.
\newblock In \emph{ICLR}, 2016.

\bibitem[Deng et~al.(2009)Deng, Dong, et~al.]{imagenet_cvpr09}
J.~Deng, W.~Dong, et~al.
\newblock {ImageNet: A Large-Scale Hierarchical Image Database}.
\newblock In \emph{CVPR}, 2009.

\bibitem[Denton et~al.(2014)Denton, Zaremba, et~al.]{Denton2014Exploiting}
Emily Denton, Zaremba, et~al.
\newblock Exploiting linear structure within convolutional networks for
  efficient evaluation.
\newblock In \emph{NIPS}, 2014.

\bibitem[Everingham et~al.(2010)Everingham, Gool, et~al.]{everingham2010pascal}
Mark Everingham, Luc~Van Gool, et~al.
\newblock The pascal visual object classes (voc) challenge.
\newblock \emph{IJCV}, 88\penalty0 (2), 2010.

\bibitem[Girshick et~al.(2014)Girshick, Donahue, et~al.]{rcnn14}
Ross Girshick, Jeff Donahue, et~al.
\newblock Rich feature hierarchies for accurate object detection and semantic
  segmentation.
\newblock In \emph{CVPR}, 2014.

\bibitem[Han et~al.(2015)Han, Pool, et~al.]{han2015learning}
Song Han, Jeff Pool, et~al.
\newblock Learning both weights and connections for efficient neural network.
\newblock In \emph{NIPS}, 2015.

\bibitem[Han et~al.(2016)Han, Mao, and Dally]{han2015deep}
Song Han, Huizi Mao, and Bill Dally.
\newblock Deep compression: Compressing deep neural networks with pruning,
  trained quantization and huffman coding.
\newblock In \emph{NIPS}, 2016.

\bibitem[He et~al.(2016)He, Zhang, et~al.]{he2016deep}
K.~He, X.~Zhang, et~al.
\newblock Deep residual learning for image recognition.
\newblock In \emph{CVPR}, 2016.

\bibitem[He et~al.(2017)He, Zhang, and Others]{He2017Channel}
Y.~He, X.~Zhang, and Others.
\newblock Channel pruning for accelerating very deep neural networks.
\newblock In \emph{ICCV}, 2017.

\bibitem[Hinton et~al.(2015)Hinton, Vinyals, et~al.]{hinton2015distilling}
G.~Hinton, O.~Vinyals, et~al.
\newblock Distilling the knowledge in a neural network.
\newblock \emph{arXiv preprint arXiv:1503.02531}, 2015.

\bibitem[Howard et~al.(2017)Howard, Zhu, et~al.]{howard2017mobilenets}
Andrew~G Howard, Menglong Zhu, et~al.
\newblock Mobilenets: Efficient convolutional neural networks for mobile vision
  applications.
\newblock \emph{arXiv preprint arXiv:1704.04861}, 2017.

\bibitem[Jaderberg et~al.(2014)Jaderberg, Vedaldi,
  et~al.]{Jaderberg2014Speeding}
M.~Jaderberg, A.~Vedaldi, et~al.
\newblock Speeding up convolutional neural networks with low rank expansions.
\newblock In \emph{BMVC}, 2014.

\bibitem[Jia et~al.(2014)Jia, Shelhamer, et~al.]{jia2014caffe}
Yangqing Jia, Evan Shelhamer, et~al.
\newblock Caffe: Convolutional architecture for fast feature embedding.
\newblock In \emph{ACM Multimedia}, 2014.

\bibitem[Kim and Wong(2017)]{batselier2017constructive}
B.~Kim and N.~Wong.
\newblock A constructive arbitrary-degree kronecker product decomposition of
  tensors.
\newblock \emph{Numerical Linear Algebra with Applications}, 24\penalty0 (5),
  2017.

\bibitem[Kim et~al.(2016)Kim, Park, et~al.]{kim2015compression}
Y.~Kim, E.~Park, et~al.
\newblock Compression of deep convolutional neural networks for fast and low
  power mobile applications.
\newblock In \emph{ICLR}, 2016.

\bibitem[Krizhevsky and Hinton(2012)]{alexnet12}
A.~Krizhevsky and G.~Hinton.
\newblock Imagenet classification with deep convolutional neural networks.
\newblock In \emph{NIPS}, 2012.

\bibitem[Li et~al.(2017)Li, Kadav, et~al.]{Li2016Pruning}
Hao Li, Asim Kadav, et~al.
\newblock Pruning filters for efficient convnets.
\newblock In \emph{ICLR}, 2017.

\bibitem[Liu et~al.(2017{\natexlab{a}})Liu, Zoph, et~al.]{liu2017pnas}
C.~Liu, B.~Zoph, et~al.
\newblock Progressive neural architecture search.
\newblock \emph{arXiv preprint arXiv:1712.00559}, 2017{\natexlab{a}}.

\bibitem[Liu et~al.(2016)Liu, Anguelov, et~al.]{liu2016ssd}
Wei Liu, Dragomir Anguelov, et~al.
\newblock Ssd: Single shot multibox detector.
\newblock In \emph{ECCV}, 2016.

\bibitem[Liu et~al.(2017{\natexlab{b}})Liu, Li, et~al.]{liu2017learning}
Zhuang Liu, Jianguo Li, et~al.
\newblock Learning efficient convolutional networks through network slimming.
\newblock \emph{arxiv preprint}, 1708, 2017{\natexlab{b}}.

\bibitem[Loan(2000)]{van2000ubiquitous}
CF~Van Loan.
\newblock The ubiquitous kronecker product.
\newblock \emph{Journal of computational and applied mathematics}, 123, 2000.

\bibitem[Long et~al.(2015)Long, Shelhamer, et~al.]{fcn15}
Jonathan Long, Evan Shelhamer, et~al.
\newblock Fully convolutional networks for semantic segmentation.
\newblock In \emph{CVPR}, 2015.

\bibitem[Luo et~al.(2017)Luo, Wu, et~al.]{luo2017thinet}
J.~Luo, J.~Wu, et~al.
\newblock Thinet: A filter level pruning method for deep neural network
  compression.
\newblock In \emph{ICCV}, 2017.

\bibitem[Rastegari et~al.(2016)Rastegari, Ordonez, et~al.]{rastegari2016xnor}
M.~Rastegari, V.~Ordonez, et~al.
\newblock Xnor-net: Imagenet classification using binary convolutional neural
  networks.
\newblock In \emph{ECCV}, 2016.

\bibitem[Ren et~al.(2017)Ren, He, et~al.]{Ren2017Faster}
S.~Ren, K.~He, et~al.
\newblock Faster r-cnn: Towards real-time object detection with region proposal
  networks.
\newblock \emph{IEEE TPAMI}, 39\penalty0 (6), 2017.

\bibitem[Simonyan and Zisserman(2015)]{simonyan2014very}
Karen Simonyan and Andrew Zisserman.
\newblock Very deep convolutional networks for large-scale image recognition.
\newblock In \emph{ICLR}, 2015.

\bibitem[Wen et~al.(2016)Wen, Wu, et~al.]{Wen2016Learning}
W.~Wen, C.~Wu, et~al.
\newblock Learning structured sparsity in deep neural networks.
\newblock In \emph{NIPS}, 2016.

\bibitem[Zhang et~al.(2017{\natexlab{a}})Zhang, Qi, et~al.]{zhang2017igc}
T.~Zhang, G.~Qi, et~al.
\newblock Interleaved group convolutions.
\newblock In \emph{ICCV}, 2017{\natexlab{a}}.

\bibitem[Zhang et~al.(2016)Zhang, Zou, et~al.]{Zhang2016Accelerating}
X.~Zhang, J.~Zou, et~al.
\newblock Accelerating very deep convolutional networks for classification and
  detection.
\newblock \emph{IEEE TPAMI}, 38\penalty0 (10), 2016.

\bibitem[Zhang et~al.(2017{\natexlab{b}})Zhang, Zhou,
  et~al.]{zhang2017shufflenet}
X.~Zhang, X.~Zhou, et~al.
\newblock Shufflenet: An extremely efficient convolutional neural network for
  mobile devices.
\newblock \emph{arXiv preprint arXiv:1707.01083}, 2017{\natexlab{b}}.

\bibitem[Zoph and Le(2016)]{zoph2016neural}
B.~Zoph and Quoc~V Le.
\newblock Neural architecture search with reinforcement learning.
\newblock \emph{arXiv preprint arXiv:1611.01578}, 2016.

\end{thebibliography}
}
\end{document}